\documentclass[11pt]{myclass}
\usepackage{euscript}

\usepackage{epsfig}
\usepackage{color}
\usepackage{bbm}
\usepackage{wrapfig}

\newcommand{\eps}{\varepsilon}
\newcommand{\disc}{\mathsf{disc}}

\DeclareMathOperator{\kde}{\textsc{kde}}
\DeclareMathOperator{\poly}{\mathsf{poly}}

\newcommand{\km}{\ensuremath{\hat{\mu}}}
\renewcommand{\H}{\ensuremath{{\Eu{H}_K}}}

\newcommand{\Eu}[1]{\ensuremath{\EuScript{#1}}}

\newcommand{\R}{\ensuremath{\mathbb{R}}}

\newcommand{\norm}[1]{\left\lVert#1\right\rVert}

\newcommand{\setdef}[2]{\left\{ #1 \mid #2 \right\}}
\newcommand{\abs}[1]{\left| #1 \right|}

\title{Near-Optimal Coresets of Kernel Density Estimates}
\author{Jeff M. Phillips\thanks{Thanks to supported by NSF CCF-1350888, IIS-1251019, ACI-1443046, CNS-1514520, and CNS-1564287.} 
	\\University of Utah 
	\and Wai Ming Tai
	\\ University of Utah}

\begin{document}

\maketitle

\begin{abstract}
	We construct near-optimal coresets for kernel density estimates for points in $\mathbb{R}^d$ when the kernel is positive definite.
	Specifically we show a polynomial time construction for a coreset of size $O(\sqrt{d}/\eps \cdot \sqrt{\log 1/\eps} )$, and we show a near-matching lower bound of size $\Omega(\min\{\sqrt{d}/\eps, 1/\eps^2\})$.  
	When $d\geq 1/\eps^2$, it is known that the size of coreset can be $O(1/\eps^2)$.
	The upper bound is a polynomial-in-$(1/\eps)$ improvement when $d \in [3,1/\eps^2)$ 
	%(for all kernels except the Gaussian kernel which had a previous upper bound of $O((1/\eps) \log^d (1/\eps))$)
	 and the lower bound is the first known lower bound to depend on $d$ for this problem. 
	Moreover, the upper bound restriction that the kernel is positive definite is significant in that it applies to a wide-variety of kernels, specifically those most important for machine learning.
	This includes kernels for information distances and the sinc kernel which can be negative.
\end{abstract}

%%%%%%%%%%%%%%%%%%%%%%%%%%%%%%%%%%%%%%%%%%%%%%%%%%%%%%%%%%%%%%%%%%%
%%%%%%%%%%%%%%%%%%%%%%%%%%%%%%%%%%%%%%%%%%%%%%%%%%%%%%%%%%%%%%%%%%%

\section{Introduction}
Kernel density estimates are pervasive objects in data analysis.  They are the classic way to estimate a continuous distribution from a finite sample of points~\cite{Sil86,Sco92}.  With some points negatively weighted, they are the prediction function in kernel SVM classifiers~\cite{SS02}.  They are the core of many robust topological reconstruction approaches~\cite{PWZ15,FLRWBS14,BMT17}.  
And they arise in many other applications including 
mode estimation~\cite{arias2016estimation}, 
outlier detection~\cite{schubert2014generalized}, 
regression~\cite{fan1996local}, 
and clustering~\cite{rinaldo2010generalized}.  

Generically, consider a dataset $P \subset \R^d$ of size $n$, and a kernel 
$K : \R^d \times \R^d \to \R$, for instance the Gaussian kernel $K(x,p) = \exp(-\alpha^2 \|x-p\|^2)$ with $1/\alpha$ as a bandwidth parameter.  Then a kernel density estimate is defined at any point $x \in \R^d$ as $\kde_P(x) = \frac{1}{n} \sum_{p \in P} K(x,p)$.  

Given that it takes $O(n)$ time to evaluate $\kde_P$, and that data sets are growing to massive sizes, in order to continue to use these powerful modeling objects, a common approach is to replace $P$ with a much smaller data set $Q$ so that $\kde_Q$ approximates $\kde_P$.  While statisticians have classically studied various sorts of average deviations ($L_2$~\cite{Sil86,Sco92} or $L_1$ error~\cite{DG84}), for most modern data modeling purposes, a worst-case $L_\infty$ is more relevant (e.g., for preserving classification margins~\cite{SS02}, density estimates~\cite{ZP15}, topology~\cite{PWZ15}, and hypothesis testing on distributions~\cite{GBRSS12}).  Specifically this error guarantee preserves
\[
\| \kde_P - \kde_Q \|_\infty = \sup_{x \in \R^d} | \kde_P(x) - \kde_Q(x) | \leq \eps.
\]
We call such a set $Q$ an \emph{$\eps$-KDE coreset} of $P$.  
In this paper we study how small can $Q$ be as a function of error $\eps$, dimension $d$, and properties of the kernels.%
\footnote{This combines results published in SOCG 2018~\cite{PT18} and SODA 2018~\cite{phillips2018improved}.}

%%%%%%%%%%%%%%%%%%%%%%%%%%%%%%%%%%%%%%%%%%%%%%%%%%%%%%%%%%%%%%%%%%%%
\subsection{Background on Kernels and Related Coresets}
\label{sec:background}

Traditionally the approximate set $Q$ has been considered to be constructed as a random sample of $P$~\cite{Sil86,Sco92,joshi2011comparing}, sometimes known as a Nystr\"{o}m approximation~\cite{DM05}.  However, in the last decade, a slew of data-aware approaches have been developed that can obtain a set $Q$ with the same $L_\infty$ error guarantee, but with considerably smaller size.  

To describe the random sample results and the data-aware approaches, we first need to be more specific about the properties of the kernel functions.  We start with \emph{positive definite} kernels, the central class required for most machine learning approaches to work~\cite{HSS06}.

%%%%%%%%%%%%%%%%%%%%%%%%%%%%%%%%%%%%
\paragraph{Postive definite kernels.}
Consider a kernel $K : \Eu{D} \times \Eu{D} \to \R$ defined over some domain $\Eu{D}$ (often $\R^d$).  It is called a \emph{positive definite kernel} if any $m$ points $x_1, x_2, \ldots, x_m \in \Eu{D}$ are used to define an $m \times m$ Gram matrix $G$ so each $i,j$ entry is $G_{i,j} = K(x_i, x_j)$, and the matrix $G$ is positive definite.  Recall, a symmetric matrix $G$ is positive definite if any vector $z \in \R^n$ that is not all zeros satisfies $z^T G z > 0$.  
Moreover, a positive definite matrix $G$ can always be decomposed as a product $H^T H$ with real-valued matrix $H$.  

Also
if $K$ is positive definite, it is said to have the reproducing property~\cite{aronszajn1950theory,Wah99}.
This implies that $K(p,x)$ is an inner product in a reproducing kernel Hilbert space (RKHS) $\H$.  Specifically, there exists a lifting map $\phi_K : \mathbb{R}^d \to \H$ where $\phi_K(x) = K(x, \cdot)$ and so $K(p,x) = \langle \phi_K(p), \phi_K(x) \rangle_{\H}$.  Moreover the entire set $P$ can be represented as $\Phi_K(P) = \sum_{p \in P} \phi_K(p)$, which is a single element of $\H$ and has norm $\|\Phi_K(P)\|_{\H} = \sqrt{\sum_{p \in P} \sum_{p' \in P} K(p,p')}$.  A single point $x \in \mathbb{R}^d$ also has a norm $\|\phi_K(x)\|_{\H} = \sqrt{K(x,x)} = 1$ in this space.  
A \emph{kernel mean} of a point set $P$ and a reproducing kernel $K$ is defined 
$
\km_P = \frac{1}{|P|} \sum_{p \in P} \phi_K(p) = \Phi_K(P)/|P| \in \H.
$

\begin{wraptable}{r}{95mm}	
	\centering{\textsf{Example Positive Definite Kernels}} 
	\begin{tabular}{r  c  c}
		\hline 
		& $K(x,p)=$ & domain \\
		\hline
		Gaussian & $\exp(-\alpha^2 \|x-p\|^2)$ & $\R^d$
		\\
		Laplacian & $\exp(-\alpha \|x-p\|)$ & $\R^d$
		\\
		Exponential & $\exp(-\alpha (1-\langle x,p \rangle))$ & $\mathbb{S}^d$
		\\
		JS & $\exp(-\alpha (H(\frac{x+p}{2}) - \frac{H(x) + H(p)}{2}))$ & $\Delta^d$
		\\
		Helinger & $\exp(-\alpha \sum_{i=1}^d (\sqrt{x_i} - \sqrt{p_i})^2)$ & $\Delta^d$
		\\
		Sinc & $\frac{\sin(\alpha \|x-p\|)}{\alpha \|x-p\|}$ & $\R^{d \leq 3}$
		\\
		\hline
	\end{tabular}
\end{wraptable}

There are many positive definite kernels, and we will next highlight a few.  
We normalize all kernels so $K(x,x) = 1$ for all $x \in \Eu{D}$ and therefore $\abs{K(x,y)} \leq 1$ for all $x,y \in \Eu{D}$.  
We will use $\alpha > 0$ as a parameter, where $1/\alpha$ represents the bandwidth, or smoothness of the kernel.  
For $\Eu{D} = \R^d$ the most common positive definite kernels~\cite{SGFSL10} are the Gaussian (described earlier) and the Laplacian, defined $\exp(-\alpha \|x-y\|)$ for $x,y \in \R^d$.  
Another common domain is $\Delta^d = \{x \in \R^{d+1} \mid \sum_{i=1}^d x_i = 1, \; x_i \geq 0\}$, for instance in representing discrete distributions such as normalized counts of words in a text corpus or fractions of tweets per geographic region.  Common positive definite kernels for $x,y \in \Delta^d$ include 
the Hellinger kernel $\exp(-\alpha \sum_{i=1}^d(\sqrt{x_i}-\sqrt{y_i})^2)$ and 
the Jensen-Shannon (JS) divergence kernel $\exp(-\alpha(H(\frac{x+y}{2}) - \frac{H(x)+H(y)}{2}))$, where $H(x) = \sum_{i=1}^d - x_i \log x_i$ is entropy~\cite{hein2005hilbertian}.  
In other settings it is more common to normalize data points $x$ to lie on a sphere $\mathbb{S}^d = \{x \in \R^{d+1} \mid \|x\| = 1\}$.  Then with $x,y \in \mathbb{S}^d$, the exponential kernel $\exp(-\alpha (1-\langle x, y \rangle))$ is positive definite~\cite{HSS06}.  
Perhaps surprisingly, positive definite kernels do not need to satisfy $K(x,y) \geq 0$.  For $x,y \in \R^d$, the sinc kernel is defined as $\frac{\sin(\alpha \|x-y\|)}{\alpha \|x-y\|}$ and is positive definite for $d = \{1,2,3\}$~\cite{schoenberg1938metric}.

%%%%%%%%%%%%%%%%%%%%%%%%%%%%%%%%%%%%
\paragraph{Other classes of kernels.}
There are other ways to characterize kernels, which provide sufficient conditions for various other coreset bounds.  For clarity, we describe these for kernels with a domain of $\R^d$, but they can apply more generally.  

We say a kernel $K$ is \emph{$C_K$-Lipschitz} if, for any $x,y,z \in \R^d$, $|K(x,z) - K(y,z)| \leq C_K \cdot \|x-y\|$.  This ensures that the kernels do not fluctuate too widely, a necessity for robustness, but also prohibits ``binary'' kernels; for instance the binary \emph{ball kernel}  is defined 
$K(x,y) = \{1 \text{ if } \|x-y\| \leq r; \text{ and } 0 \text{ otherwise}\}$.  Such binary kernels are basically range counting queries (for instance the ball kernel corresponds with a range defined by a ball), and as we will see, this distinction allows the bounds for $\eps$-KDE coresets to surpass lower bounds for coresets for range counting queries.  
Aside from the ball kernel, all kernels we discuss in this paper will be $C_K$-Lipschitz. 

Another way to characterize a kernel is with their shape.  We can measure this by considering binary ranges defined by super-level sets of kernels.  For instance, given a fixed $K$ and $x \in \R^d$, and a threshold $\tau \in \R$, the \emph{super-level set} is $\{p \in \R^d \mid K(x,p) \geq \tau\}$.  For a fixed $K$, the family of such sets over all choices of $x$ and $\tau$ describes a range space with ground set $\R^d$.  For many kernels the VC-dimension of this range space is bounded; in particular, for common kernels, this range is equivalent to those defined by balls in $\R^d$.  Notably, the sinc kernel, which is positive-definite for $\R^d$ with $d \leq 3$ does not correspond to a range space with bounded VC-dimension.  

Finally, we mention that kernels being \emph{characteristic}~\cite{SGFSL10} is an important property for many bounds that rely on $\phi_K$.  
It includes most, but not all positive definite kernels including Gaussian and Laplace kernels; the notable exceptions are the Euclidean dot product $\langle x, y \rangle$, and anything derivative of it such as the exponential kernel.
%\waiming{It's unclear to me what is meant by "the mapping $\phi_K(x)$ is isomorphic". The definition of a characteristic kernel I have seen is that the mapping from distributions to kernel means should be bijective.}
A characteristic kernel requires that the kernel $K$ is positive definite, and the mapping $\phi_K(x)$ is injective and ultimately this implies its induced distance 
\[
D_K(p,x) = \sqrt{\|\phi_K(x)\|_{\mathcal{H}_K}^2 + \|\phi_K(p)\|_{\mathcal{H}_K}^2 - 2\langle\phi_K(p),\phi_K(x)\rangle_{\mathcal{H}_K}}
\] 
is a metric~\cite{muller1997integral,SGFSL10}.  

%%%%%%%%%%%%%%%%%%%%%%%%%%%%%%%%%%%%
\paragraph{Kernel distance.}
This $D_K$ is known as the \emph{kernel distance}~\cite{hein2005hilbertian,glaunesthesis,
	joshi2011comparing,PhillipsVenkatasubramanian2011} (or \emph{current distance} or \emph{maximum mean discrepancy}).
If we define the similarity between the two point sets as 
\[
\kappa(P,Q) = \frac{1}{|P|}\frac{1}{|Q|} \sum_{p \in P} \sum_{q \in Q} K(p,q) = \langle\km_P,\km_Q\rangle_{\mathcal{H}_K},
\]
then the kernel distance can be defined more generally between point sets (implicitly endowed with uniform probability measures) as 
\[
D_K(P,Q) = \sqrt{\kappa(P,P) + \kappa(Q,Q) - 2 \kappa(P,Q)} = \|\km_P - \km_Q\|_{\mathcal{H}_K}.
\]
When $Q$ is a single point $x$, then $\kappa(P,x) = \langle\km_P,\phi_K(x)\rangle_{\mathcal{H}_K} =\kde_P(x)$.

%%%%%%%%%%%%%%%%%%%%%%%%%%%%%%%%%%%%
\paragraph{Relationship between kernel mean and $\eps$-KDE coresets.}
It is possible to convert between bounds on the subset size required to approximate the kernel mean and an $\eps$-kernel coreset of an associated kernel range space.  But they are not symmetric.  

%Note that $\|\km_P\|_{\H} \leq K(x,x)$ for any $P$, so in our setting $\|\km_P\|_{\H} \leq 1$.  

The Koksma-Hlawka inequality (in the context of reproducing kernels~\cite{CWS10,SZSGS08} when $K(x,x) = 1$) states that
\[
\|\kde_P - \kde_Q\|_\infty \leq  \|\km_P - \km_Q\|_{\Eu{H}_K}.
\]
Since $\kde_P(x) = \kappa(P, x) = \langle \km_P, \phi(x) \rangle_\H$ and via Cauchy-Schwartz, for any $x \in \mathbb{R}^d$  
\[
| \kde_P(x) - \kde_Q(x) |
=
| \langle \km_P , \phi(x) \rangle_\H - \langle \km_Q, \phi(x) \rangle_\H |
= 
| \langle \km_P - \km_Q, \phi(x) \rangle_\H |
\leq 
\| \km_P - \km_Q \|_\H.  
\]
Thus to bound $\max_{x \in \mathbb{R}^d} |\kde_P(x) - \kde_Q(x)| \leq \eps$ it is sufficient to bound $\|\km_P - \km_Q\|_\H \leq \eps$.  

On the other hand, if we have a bound $\max_{x \in \mathbb{R}^d} |\kde_P(x) - \kde_Q(x)| \leq \eps$, then we can only argue that $\|\km_P - \km_Q\|_\H \leq \sqrt{2 \eps}$.  
We observe that
\begin{align*}
\|\km_P - \km_Q\|_\H^2 
&= 
D_K(P,Q)^2 = \kappa(P,P) + \kappa(Q,Q) - 2 \kappa(P,Q)
\\ & =
\frac{1}{|P|} \sum_{p \in P} \kde_P(p) + \frac{1}{|Q|} \sum_{q \in Q} \kde_Q(q)  - \frac{1}{|P|} \sum_{p \in P} \kde_Q(p)  - \frac{1}{|Q|} \sum_{q \in Q} \kde_P(q)
\\ &=
\frac{1}{|P|} \sum_{p \in P} (\kde_P(p) - \kde_Q(p)) + \frac{1}{|Q|} \sum_{q \in Q} (\kde_Q(q) - \kde_P(q))
\\ & \leq
\frac{1}{|P|} \sum_{p \in P} (\eps) + \frac{1}{|Q|} \sum_{q \in Q} (\eps)
= 2\eps.  
\end{align*}
%We can also take the only inequality the other direction to get the lower bound.  
Taking a square root of both sides leads to the implication.  

Unfortunately, the second reduction does not map the other way; a bound on $\|\km_P - \km_Q\|_H^2$ only ensures an average ($L_2$ error) for $\kde_P$ holds, not the desired stronger $L_\infty$ error.  
Indeed, the above bound is tight.   Consider $P$ and $Q$ so $|P| = |Q| = 1/\eps$, and all pairs of points $x,y \in P \cup Q$ are sufficiently far away from each other so $K(x,y) \leq \eps^2$, and hence it must be that $\| \kde_P - \kde_Q \|_\infty \leq \eps^2$.  
However, then we can also bound
\begin{align*}
	\|\km_P-\km_Q\|_\H^2
	& =
	\frac{1}{\abs{P}}\frac{1}{\abs{Q}}\left( \sum_{p \in P} \sum_{p'\in P}K(p,p') - 2\sum_{p\in P}\sum_{q\in Q}K(p,q) + \sum_{q \in Q}\sum_{q'\in Q}K(q,y) \right) \\
	&  \geq
	\frac{1}{\abs{P}}\frac{1}{\abs{Q}}\left( |P| + (|P| -1)^2 0  - 2|P| |Q| \eps^2 + |Q| + (|Q|-1)^2 0 \right) \\
	&  =
	\eps^2\left(  \abs{P} - 2\abs{P}\abs{Q}\eps^2   +\abs{Q}  \right) \\
	& =
	2 \eps - 2 \eps^2 = \Omega(\eps).  
\end{align*}

%%%%%%%%%%%%%%%%%%%%%%%%%%%%%%%%%%%%
\paragraph{Discrepancy-based approaches.}
Our approach for creating an $\eps$-KDE coreset will follow a technique for creating range counting coresets~\cite{CM96,Phi08,BS80}.  It focuses on assigning a coloring $\chi : P \to \{-1,+1\}$ to $P$.  Then retains either all $P_+ = \{p \in P \mid \chi(p) = +1\}$ or the remainder $P_-$, and recursively applies this halving until a small enough coreset $Q$ has been retained.  

Classically, when the goal is to compute a range counting coreset for a range space $(P, \Eu{R})$, then the specific goal of the coloring is to minimize discrepancy
\[
\disc_R(P,\chi) =  \abs{\sum_{p\in P \cap R} \chi(p)}
\]
over all choices of ranges $R \in \Eu{R}$.  In the KDE-setting we consider a \emph{kernel range space} $(P,\Eu{K})$ where $\Eu{K} = \{K(x,\cdot) \mid x \in \Eu{D}\}$ defined by kernel $K : \Eu{D} \times \Eu{D} \to \R$ and a fixed domain $\Eu{D}$ which is typically assumed, and usually $\Eu{D} = \R^d$.  
We instead want to minimize the kernel discrepancy
\[
\disc(P,\chi,x) =  \abs{\sum_{p\in P} \chi(p) K(x,p)}.
\]
Now in contrast to the case with the binary range space $(P,\Eu{R})$, \emph{each} point $p \in P$ is \emph{partially} inside the ``range'' where the amount inside is controlled by the kernel $K$.  
Understanding the quantity 
\[
\disc(n,\Eu{K}) = \max_{P : |P| = n} \min_\chi \max_{x \in \Eu{D}} \disc(P,\chi, x)
\]
is key.  If for a particular $\Eu{K}$ we have $\disc(n,\Eu{K}) = n^{\tau}$ or $\disc(n, \Eu{K}) = \log^\eta n$, then applying the recursive halving algorithm obtains an $\eps$-KDE coreset of size $O(1/\eps^{1/(1-\tau)})$ and $O((1/\eps) \log^\eta (1/\eps))$, respectively~\cite{phillips2013varepsilon}.

%Given a point set $P$ of $n$ points in $\mathbb{R}^d$.
%We want to find a subset $Q\subset P$ of size $k$ such that, given $\eps > 0$ and for any $x \in \mathbb{R}^d$, $\abs{\kde_P(x)-\kde_Q(x)} < \eps$.
%Such $Q$ is called $\eps$-kernel coreset.
%Our goal is to minimize $k$, the size of $Q$.
%
%By standard halving technique, it is equivalent to the following discrepancy problem.
%Denote $\mathcal{K}_d = \setdef{K(\cdot, x)}{x\in \mathbb{R}^d}$.
%Suppose $\chi: P \rightarrow \{-1,+1\}$ be a coloring on $P$.
%Define, for any $P \in \mathbb{R}^d$, $\chi$ and $x \in \mathbb{R}^d$,
%\[
%\disc(P,\chi, x) =  \abs{\sum_{p\in P} \chi(p) K(p,x) }
%\]
%and
%\[
%\disc(n, \mathcal{K}_d) = \max_{\abs{P}=n}\min_{\chi}\max_{x\in \mathbb{R}^d} \disc(P,\chi,x)
%\]
%The goal is to minimize $\disc(n, \mathcal{K}_d)$.

%%%%%%%%%%%%%%%%%%%%%%%%%%%%%%%%%%%%%%%%%%%%%%%%%%%%%%%%%%%%%%%%%%%%
\subsection{Known Results on KDE Coresets}
\label{sec:previous}
In this section we survey known bounds on the size $|Q|$ required for $Q$ to be an $\eps$-KDE coreset.  
We assume $P \subset \mathbb{R}^d$, it is of size $n$, and $P$ has a diameter 
$\Delta  = \alpha \max_{p, p' \in P} \|p-p'\|$, 
where $1/\alpha$ is the bandwidth parameter of the kernel.  
We sometimes allow a $\delta$ probability that the algorithm does not succeed.  
Results are summarized in Table \ref{tbl:compare}.

\begin{table}[t]
	\begin{tabular}{r c l l}
		\hline
		Paper & Coreset Size & Restrictions & Algorithm
		\\ \hline 
		Joshi \etal~\cite{joshi2011comparing} & $d/\eps^2$ &  bounded VC & random sample
		\\
		Fasy \etal~\cite{FLRWBS14} &  $(d/\eps^2) \log(d \Delta / \eps)$ & Lipschitz & random sample
		\\ 
%		\waiming{check}Gretton \etal~\cite{GBRSS12} & $1/\eps^4$ & characteristic kernels & random sample
%		\\
		Lopaz-Paz \etal~\cite{LopazPaz15} & $1/\eps^2$ & characteristic kernels & random sample
		\\
		Chen \etal~\cite{CWS10} & $1/(\eps r_P)$ & characteristic kernels & iterative
		\\ 
		Bach \etal~\cite{bach2012equivalence} & $(1/r_P^2)\log (1/\eps)$ & characteristic kernels & iterative  
		\\
		Bach \etal~\cite{bach2012equivalence} & $1/\eps^2$  & characteristic kernels, weighted & iterative
		\\
		Lacsote-Julien \etal~\cite{lacoste2015sequential} & 	$1/\eps^2$  & characteristic kernels& iterative\\
		Harvey  and Samadi~\cite{harvey2014near} & $(1/\eps)\sqrt{n}\log^{2.5}(n)$  & characteristic kernels & iterative
		\\
		Cortez and Scott \cite{CS15} & $k_0$ \;\; ($\leq (\Delta/ \eps)^d$) & Lipschitz; \; $d$ is constant & iterative
		\\
		Phillips~\cite{phillips2013varepsilon} & $(1/\eps)^{\frac{2d}{d+2}} \log^{\frac{d}{d+2}} (1/\eps)$ & Lipschitz; \; $d$ is constant & discrepancy-based
		\\
		Phillips~\cite{phillips2013varepsilon} & $\Theta(1/\eps)$ & $d=1$ & sorting
		%		\\ 
		%		Phillips and Tai~\cite{phillips2018improved} &  $(1/\eps) \log^d (1/\eps)$ & Gaussian, $d$ is constant; discrepancy-based
		%		\\
		%		Phillips and Tai~\cite{phillips2018improved} &  $\Omega(1/\eps^2)$ & \textsc{siss} (e.g., Gaussian); $d = \Omega(1/\eps^2)$  
		\\\hline
	\end{tabular} 
	\caption{Asymptotic $\eps$-KDE coreset sizes in terms of error $\eps$ and dimension $d$.   %We set probability of failure $\delta$, and bandwidth $1/\alpha$, as a constant.  
%		\textsc{siss} = Shift- and rotation invariant, and somewhere-steep (see Section \ref{lower}).    
		\label{tbl:compare}}
%	\vspace{-.2in}
\end{table}

\paragraph{Halving approaches.}
Phillips~\cite{phillips2013varepsilon} showed that kernels with a bounded Lipschitz factor (so $|K(x,p) - K(x,q)| \leq C \|p-q\|$ for some constant $C$, including Gaussian, Laplace, and Triangle kernels which have $C = O(\alpha)$), admit coresets of size $O((\alpha/\eps) \sqrt{\log(\alpha/\eps)})$ in $\mathbb{R}^2$.  For points in $\mathbb{R}^d$ (for $d>1$) this generalizes to a bound of $O((\alpha/\eps)^{2d/(d+2)} \log^{d/(d+2)} (\alpha/\eps))$.  
That paper also observed that for $d=1$, selecting evenly spaced points in the sorted order achieves a coreset of size $O(1/\eps)$.  

\paragraph{Sampling bounds.}
Denote $\delta$ to be the failure probability.
Joshi \etal~\cite{joshi2011comparing} showed that a random sample of size $O((1/\eps^2)(d + \log(1/\delta)))$ results in an $\eps$-kernel coreset for any centrally symmetric, non-increasing kernel.  This works by reducing to a VC-dimensional~\cite{LLS01} argument with ranges defined by balls.  

Fasy~\etal~\cite{FLRWBS14} %ArXiv:1303.7117
provide an alternative bound on how random sampling preserves the $L_\infty$ error in the context of statistical topological data analysis.  Their bound can be converted to require size  
$O((d/\eps^2) \log(d \Delta/\eps \delta))$, which can improve upon the bound of Joshi~\cite{joshi2011comparing} if $K(x,x) > 1$ (otherwise, herein we only consider the case $K(x,x) =1$).  

Examining characteristic kernels which induce an RKHS in that function space leads to a simpler bound of $O((1/\eps^2) \log(1/\delta))$~\cite{MFSS17}; see Lopaz-Paz \etal~\cite{LopazPaz15} for a simple and complete proof.

\paragraph{Iterative approaches.}
Motivated by the task of constructing samples from Markov random fields, Chen \etal~\cite{CWS10} introduced a technique called \emph{kernel herding} suitable for characteristic kernels.  
They showed that iteratively and greedily choosing the point $p \in P$ which when added to $Q$ most decreases the quantity $\|\km_P - \km_Q\|_{\Eu{H}_K}$, will decrease that term at rate $O(r_P/t)$ for $t = |Q|$.  Here $r_P$ is the largest radius of a ball centered at $\km_P \in \H$ which is completely contained in the convex hull of the set $\{\phi(p) \mid p \in P\}$.  
They did not specify the quantity $r_P$ but claimed that it is a constant greater than $0$.

Bach \etal~\cite{bach2012equivalence} showed that this algorithm can be interpreted under the Frank-Wolfe framework~\cite{Cla10,gartner2009coresets}.  Moreover, they argue that $r_P$ is not always a constant; in particular when $P$ is infinite (e.g., it represents a continuous distribution) then $r_P$ is arbitrarily small.  
However, when $P$ is finite, they prove that $1/r_P$ is finite without giving an explicit bound.  
They also make explicit that after $t$ steps, they achieve $\|\km_P - \km_{Q,w}\|_{\Eu{H}_K} \leq 4 / (r_P \cdot t)$.  They also describe a method which includes ``line search'' to create a weighted coreset $(Q,w)$, so each point $q \in Q$ is associated with a weight $w(q) \in [0,1]$ so $\sum_{q \in Q} w(q) = 1$; then $\km_{Q,w} = \sum_{q \in Q} w(q) \phi(q)$.  For this method they achieve 
$
\|\km_P - \km_{Q,w}\|_{\Eu{H}_K} \leq \sqrt{\exp(-r_P^2 t)}.
$
Similarly, other recent progress in Frank-Wolfe analysis focuses on settings which achieve a ``linear'' rate of roughly $O(c^{-t})$~\cite{JL15,FG16}.  
However, such faster linear convergence, unless some specific properties of the data exist, would violate our lower bound, and thus is not possible in general.  

Bach \etal~\cite{bach2012equivalence} also mentions a bound $\|\km_P - \km_{Q,w}\|_{\Eu{H}_K} \leq \sqrt{8 / t}$, that is independent of $r_P$.  It relies on very general bound of Dunn~\cite{dunn1980convergence} which uses line search, or one of Jaggi~\cite{Jag13} which uses a fixed but non-uniform set of weights.  These show this convergence rate for any smooth function, including $\|\km_P - \km_{Q,w}\|^2_{\Eu{H}_K}$; taking the square root provides a bound for $\|\km_P - \km_{Q,w}\|_{\Eu{H}_K} \leq \eps$ after $t = O(1/\eps^2)$ steps.  
This result is a weighted coreset, and it has been further improved to be unweighted~\cite{lacoste2015sequential}.

Harvey and Samadi ~\cite{harvey2014near} further revisited kernel herding in the context of a general mean approximation problem in $\mathbb{R}^{d'}$.  That is, consider a set $P'$ of $n$ points in $\mathbb{R}^{d'}$, find a subset $Q' \subset P'$ so that $\|\bar P' - \bar Q'\| \leq \eps$, where $\bar P'$ and $\bar Q'$ are the Euclidean averages of $P'$ and $Q'$, respectively.  This maps to the kernel mean problem with $P' = \{\phi_K(p) \mid p \in P\}$, and with the only bound of $d'$ as $n$.  
They show that the $r_P$ term can be manipulated by affine scaling, but that in the worst case (after such transformations via John's theorem) it is $O(\sqrt{d'} \log^{2.5} (n))$, and hence 
show one can always set $\eps = O(\sqrt{d'} \log^{2.5} (n) / t) = O((1/t) \sqrt{n} \log^{2.5}(n))$.  Lacsote-Julien \etal~\cite{lacoste2015sequential} showed that one can always compress $P'$ to another set $P''$ of size $n = O(1/\eps^2)$ (or for instance use the random sampling bound of Lopaz-Paz \etal~\cite{LopazPaz15}, ignoring the $\log(1/\delta)$ factor); then solving for $t$ yields $t = O((1/\eps^2) \log^{2.5}(1/\eps))$.  

Harvey and Samadi also provide a lower bound to show that after $t$ steps, the kernel mean error may be as large as $\Omega(\sqrt{d'}/t)$ when $t = \Theta(n)$.  
This seems to imply (using the $d' = \Omega(n)$ and a $P'$ of size $\Theta(1/\eps^2)$) that we need $t = \Omega(1/\eps^2)$ steps to achieve $\eps$-error for kernel density estimates.  But this would contradict the bound of Phillips~\cite{phillips2013varepsilon}, which for instance shows a coreset of size $O((1/\eps) \sqrt{\log (1/\eps)})$ in $\mathbb{R}^2$.  More specifically, it uses $t = \Theta(d')$ steps to achieve this case, so if $d' = n = \Theta(1/\eps^2)$ then this requires asymptotically as many steps as there are points.  Moreover, a careful analysis of their construction shows that the corresponding points in $\mathbb{R}^d$ (using an inverse projection $\phi_K^{-1} : \Eu{H}_K \to \mathbb{R}^d$ to a set $P \in \mathbb{R}^d$) would have them so spread out that $\kde_P(x) < c/\sqrt{n}$ (for constant $c$, so $= O(\eps)$ for $n = 1/\eps^2$) for all $x \in \mathbb{R}^d$; hence it is easy to construct a $2/\eps$ size $\eps$-kernel coreset for this point set.  
This distinction between bounds is indeed related to the difference between kernel mean approximations and $\eps$-KDE coreset approximations.

\paragraph{Discretization bounds.}
Another series of bounds comes from the Lipschitz factor of the kernels:  $C = \max_{x,y,z \in \mathbb{R}^d} \frac{K(z,x) - K(z,y)}{\|x-y\|}$.  
For most kernels, $C$ is small constant.  
%This implies that $\max_{x,y \in \mathbb{R}^d} \frac{\kde_P(x) - \kde_P(y)}{\|x-y\|} \leq C$ for any $P$.  
Thus, we can for instance, lay down an infinite grid $G_{\eps} \subset \mathbb{R}^d$ of points so for all $x \in \mathbb{R}^d$ there exists some $g \in G_\eps$ such that $\|g-x\| \leq \eps/C$, and  that means the side length of the grid is $2\eps/(C\sqrt{d})$.  

Then we can map each $p \in P$ to $p_g$ the closest point $g \in G_\eps$ (with multiplicity), resulting in $P_G$.  By the additive property of $\kde$, we know that $\|\kde_P - \kde_{P_G}\|_\infty \leq \eps$.  

Cortes and Scott~\cite{CS15} provide another approach to the sparse kernel mean problem.  They run Gonzalez's algorithms~\cite{Gon85} for $k$-center on the points $P \in \mathbb{R}^d$ (iteratively add points to $Q$, always choosing the furthest point from any in $Q$) and terminate when the furthest distance to the nearest point in $Q$ is $\Theta(\eps)$.  Then they assign weights to $Q$ based on how many points are nearby, similar to in the grid argument above.  They make an ``incoherence'' based argument, specifically showing that $\|\km_P - \km_Q\| \leq \sqrt{1-v_Q}$ where $v_Q = \min_{p \in P} \max_{q \in Q} K(p,q)$.  This does not translate meaningfully in any direct way to any of the parameters we study.  However, we can use the above discretization bound to argue that if $\Delta$ is bounded, then this algorithm must terminate in $O((\Delta/\eps)^d)$ steps.

\paragraph{Lower bounds.}
Finally, there is a simple lower bound of size $\lceil 1/\eps \rceil - 1$ for an $\eps$-coreset $Q$ for kernel density estimates~\cite{phillips2013varepsilon}.  Consider a point set $P$ of size $1/\eps-1$ where each point is very far away from every other point, then we cannot remove any point otherwise it would create too much error at that location.

%%%%%%%%%%%%%%%%%%%%%%%%%%%%%%%%%%%
\subsection{Our Results}
\label{sec:our-results}

\begin{wraptable}{r}{75mm}
\vspace{-5mm}
	\begin{tabular}{r  c  c c}
		\hline 
		$d$ & Upper & Lower & 
		\\ \hline
		$1$ & $1/\eps$ & $1/\eps$ & \cite{phillips2013varepsilon}
		\\
		\hspace{-1mm}$[2,1/\eps^2)$ & $\sqrt{d}/\eps \cdot \sqrt{\log \frac{1}{\eps}}$ & $\sqrt{d}/\eps$ & \textbf{new}${^\star}$
		\\
		$\geq 1/\eps^2$ & $1/\eps^2$ & $1/\eps^2$ & \cite{lacoste2015sequential},\textbf{new}${^\dagger}$\hspace{-1mm}
		\\ \hline
	\end{tabular}
	\caption{
		\label{tbl:results} Size bounds for $\eps$-KDE coresets for Gaussian and Laplace kernels; also holds under more general assumption, see text. \;\;\;\;\;\;\;\;\;\;\;\;\;\;\;\;
		($\star$) For $d=2$, \cite{phillips2013varepsilon} matches upper bound. \;\;\;\;\;
		($\dagger$) For the lower bound result.  
	}
\end{wraptable}

We show a new upper bound on the size of an $\eps$-KDE coreset of $O((1/\eps) \sqrt{d \log (1/\eps)})$ in Section \ref{upper}.   The main restriction on the kernel $K$ is that it is positive definite, a weaker bound than the similar characteristic assumption.  There are also fairly benign restrictions (in Euclidean-like domains) that $K$ is Lipschitz and only has a value greater than $1/|P|$ (or $\geq \eps^2$) for pairs of points both within a bounded region; these are due to the specifics of some geometric preprocessing.  
Noteably, this upper bound applies to a very wide range of kernels including 
the sinc kernel, whose super-level sets do not have bounded VC-dimension and is not characteristic, so no non-trivial $\eps$-KDE coreset bound was previously known.  
Moreover, unlike previous discrepancy-based approaches, we do not need to assume the dimension $d$ is constant.  

We then show a nearly-matching lower bound on the size of an $\eps$-KDE coreset of $\Omega(\sqrt{d}/\eps)$, in Section \ref{lower}.  This construction requires a standard restriction that it is shift- and rotation-invariant, and a benign one that it is somewhere-steep (see Section \ref{lower}), satisfied by all common kernels.    This closes the problem for many kernels (e.g., Gaussians, Laplace), except for a $\sqrt{\log (1/\eps)}$ factor when $1 < d < 1/\eps^2$.  The gap filled by the new bounds are shown in Table \ref{tbl:results}.

%%%%%%%%%%%%%%%%%%%%%%%%%%%%%%%%%%
\paragraph{Our approach and context.}

Bounding the size $\eps$-KDE coresets can be reduced to bounding kernel discrepancy.  
The range space discrepancy problem, for a range space $(P,\Eu{R})$, has been widely studied in multiple areas~\cite{Mat10,Cha01}.  
For instance, Tusnady's problem restricts $\Eu{R}$ to represent axis-aligned rectangles in $\mathbb{R}^d$, has received much recent focus~\cite{matousek2014factorization}.
To achieve their result, Matousek \etal~\cite{matousek2014factorization} use a balancing technique of Banaszcyk~\cite{banaszczyk1998balancing} on a matrix version of discrepancy, by studying the so-call $\gamma_2$-norm.  

Roughly speaking, we are able to show how to directly reduce the kernel discrepancy problem to the $\gamma_2$-norm, and the bound derived from Banaszcyk's Theorem~\cite{banaszczyk1998balancing}.  In particular, the positive definiteness of a kernel, allows us to define a specific gram matrix $G$ which has a real-valued decomposition, which matches the structure studied with the $\gamma_2$ norm.  
Hence, while our positive definite restriction is similar to the characteristic restriction studied for $\eps$-KDE coresets in many other settings~\cite{GBRSS12,bach2012equivalence} it uses a very different aspect of this property: the decomposability, not the embedding.

Finally, we show a lower bound, that there exist point sets $P$ in dimension $d$, such that any $\eps$-kernel coreset requires $\Omega(\sqrt{d}/\eps)$ points.  Specifying this to the $d=1/\eps^2$ case, proves a lower bound of $\Omega(1/\eps^2)$ for any case with $d \geq 1/\eps^2$.  
This applies to every shift-invariant kernel we considered, with a slightly weakened condition for the ball kernel.

%%%%%%%%%%%%%%%%%%%%%%%%%%%%%%%%%%%%%%%%%%%%%%%%%%%%%%%%%%%%%%%%%%%
%%%%%%%%%%%%%%%%%%%%%%%%%%%%%%%%%%%%%%%%%%%%%%%%%%%%%%%%%%%%%%%%%%%
\section{Upper Bound for KDE Coreset}
\label{upper}

Recall that our result focuses on the case of $d<\frac{1}{\eps^2}$.
We assume that $P$ is finite and of size $n$; however, as mentioned in the related work, for many settings, we can reduce this to a point set of size independent of $n$ (size $1/\eps^2$ or $d/\eps^2$, depending on the kernel).  Indeed these techniques may start with inputs as continuous distributions as long as we can draw random samples or run iterative algorithm.  

Consider a point set $P \subset \R^d$ as input, but as Section \ref{sec:specific} describes, it is possible to apply these arguments to other domains.  

To prove our $\eps$-kernel coreset upper bound we introduce two properties that the kernel must have.  
\begin{itemize}
	\item We say a kernel $K$ has \emph{$c_K$-bounded influence} if, for any $x \in \mathbb{R}^d$ and $\delta>0$, $\abs{K(x,y)} < \delta$ for all $y \notin x+[-(1/\delta)^{c_K},(1/\delta)^{c_K}]^d$ for some constant $c_K$.  By default we set $\delta = 1/n$. 
	%where an associated point set has size $n$.  
	If $c_K$ is an absolute constant we simply say $K$ is \emph{bounded influence}.  
	\item We say a kernel $K$ is \emph{$C_K$-Lipschitz} if, for any $x,y,z \in \mathbb{R}^d$, $\abs{K(x,z)-K(y,z)}<C_K\norm{x-y}$ for some $C_K$.  If $C_K$ is an absolute constant within the context of the problem, we often just say the kernel is \emph{Lipschitz}.  
\end{itemize}

Next define a lattice $R = \setdef{(\frac{i_1}{\sqrt{d}n},\frac{i_2}{\sqrt{d}n}, \dots, \frac{i_d}{\sqrt{d}n})}{i_j\text{ are integers}}$.
Also, denote, for each $p \in P$,  $S_p = p+R\cap [-n^{c_K},n^{c_K}]^d$ and $S = \cup_{p\in P}S_p$.

The following lemma explains that we only need to consider the evaluation at a finite set (specifically $S$) rather than the entire space while preserving the discrepancy asymptotically. 
The advantage of doing this is we can then use a matrix representation of the discrepancy formula.  

\begin{lemma}\label{grid}
	$\max_{x\in \mathbb{R}^d} \disc(P,\chi,x) \leq \max_{x\in S} \disc(P,\chi,x)+O(1)$
\end{lemma}

\begin{proof}
	For any $x \in \mathbb{R}^d$, if $x \notin \cup_{p\in P} \left(p+[-n^{c_K},n^{c_K}]^d\right)$, that is $x$ is not within $n^{c_K}$ in all coordinates of some $p \in P$, then $K(p,x) \leq 1/n$ for all $p \in P$.  Hence we have
	\[
	\disc(P,\chi, x) =  \abs{\sum_{p\in P} \chi(p) K(p,x) } \leq O(1).
	\]
	Otherwise, pick $x_0\in S$ to be the closest point to $x$.
	We have 
	\begin{align*}
	\disc(P,\chi, x)
	& =
	\abs{\sum_{p\in P} \chi(p) K(p,x) } \\
	& =
	\abs{\sum_{p\in P} \chi(p) (K(p,x_0)+K(p,x)-K(p,x_0)) }\\
	& \leq 
	\abs{\sum_{p\in P} \chi(p) K(p,x_0) }+ \sum_{p \in P}\abs{K(p,x)-K(p,x_0)} \\
	& \leq 
	\disc(P,\chi, x_0)+ \sum_{p \in P}C_K\cdot\norm{x-x_0} \\
	& \leq 
	\disc(P,\chi, x_0)+ n\cdot C_K \cdot\sqrt{d(\frac{1}{\sqrt{d}n})^2} \\
	& =
	\disc(P,\chi, x_0)+O(1). \qedhere
	\end{align*}	
\end{proof}

Now we discuss the matrix view of discrepancy, known results, and then how to map the discretized kernel discrepancy problem into this setting.  Consider any $s\times t$ matrix $A$, and define
\[
\disc(A) = \min_{x\in \{-1,+1\}^t}\norm{Ax}_{\infty}.
\]
Following Matousek \etal~\cite{matousek2014factorization} we define $\gamma_2(A) = \min_{BC=A} r(B)\cdot c(C)$ where $r(B)$ is largest Euclidean norm of row vectors of $B$ and $c(C)$ is largest Euclidean norm of column vectors of $C$.
There is an equivalent~\cite{matousek2014factorization} geometric interpretation of $\gamma_2$.
Let $\Eu{E}_A$ be the set of ellipsoids in $\R^s$ that contain all column vectors of $A$.
Then, $\gamma_2(A) = \min_{E \in\Eu{E}_A} \max_{x\in E} \norm{x}_\infty$.
It is easy to see that $\gamma_2$ is a norm and $\gamma_2(A) \leq \gamma_2(A')$ when the columns of $A$ are subset of the columns of $A'$.  We will apply these properties shortly.  

A recent result by Matousek \etal~\cite{matousek2014factorization} shows the following property about connecting discrepancy to $\gamma_2$, which was recently made constructive in polynomial time~\cite{bansal2017gram}.  

\begin{lemma}[Matousek \etal~\cite{matousek2014factorization}]
	\label{lem:gamma2}
	For an $s\times t$ matrix $A$, 
	$
	\disc(A) \leq O(\sqrt{\log s})\cdot \gamma_2(A).  
	$
\end{lemma}

Let the size of $S$ be $m = O(n^{O(d)})$, and define an $m \times n$ matrix $G$ so its rows are indexed by $x \in S$ and columns indexed by $p \in P$, and $G_{x,p} = K(p,x)$.  By examination, $\disc(G) = \min_\chi \max_{x \in S} \disc(P,\chi,x)$.  

\begin{lemma} \label{lem:gamma2=1}
	$\gamma_2(G) = 1$.  
\end{lemma}

\begin{proof}
	Denote $G'$ be a $m\times m$ matrix with both row and column indexed $x,y \in S$ such that $G'_{x,y} = K(x,y)$.
	Note that columns of $G$ are a subset of columns of $G'$ since $P \subset S$.
	Since $K$ is a positive definite kernel, it means that $G'$ can be expressed as $H^TH$ for some matrix $H$.
	Now denote $v_x$ as the $x$th column of $H$ for all $x \in S$.
	We have $v_x^T v_x = G'_{x,x} = 1$ which means the norm $\|v_x\| = \sqrt{v_x^T v_x} = 1$ for each column $v_x \in H$.  Hence the same holds for rows in $H^T$, and this bounds $\gamma_2(G') \leq 1$.
	Then since $\gamma_2(G) \leq \gamma_2(G')$ we have $\gamma_2(G) \leq 1$.
	
	On the other hand, one of the coordinates in a column of $G$ is $1$.
	By the geometric definition, any ellipsoid containing columns of $G$ has a point inside of it such that one of its coordinates is $1$.
	Hence $\gamma_2(G) \geq 1$.	
\end{proof}

Combining all above lemmas, for any $P \subset \R^d$ of size $n$ 
\begin{align*}
\disc(n, \Eu{K}) 
& \leq 
\max_{P : |P| =n} \min_\chi \max_{x \in S} \disc(P,\chi, x) + O(1) & \text{ Lemma \ref{grid}}
\\ &= 
\max_{P : |P| =n} \disc(G) + O(1) & \text{Definition of $G$}
\\ & \leq
O(\sqrt{d \log n} \cdot \gamma_2(G)) & \text{ Lemma \ref{lem:gamma2}~\cite{matousek2014factorization}}
\\ &=
O(\sqrt{d \log n}).  & \text{ Lemma \ref{lem:gamma2=1}}  
\end{align*}

\begin{theorem}
	\label{thm:disc-up}
	Let $K: \mathbb{R}^d \times \mathbb{R}^d \rightarrow \mathbb{R}$ be a bounded influence, Lipschitz, positive definite kernel.
	For any integer $n$, $\disc(n,\mathcal{K}_d) = O(\sqrt{d\log n})$.  
\end{theorem}

\begin{corollary}
	Let $K: \mathbb{R}^d \times \mathbb{R}^d \rightarrow \mathbb{R}$ be a bounded influence, Lipschitz, positive definite kernel.
	For any set $P\subset \mathbb{R}^d$, there is a subset $Q\subset P$ of size $O(\frac{1}{\eps}\sqrt{d\log \frac{1}{\eps}})$ such that 
	\[
	\max_{x\in\mathbb{R}^d}\abs{\kde_P(x) - \kde_Q(x)} < \eps.  
	\]
\end{corollary}

\begin{proof}
	In order to apply the standard halving technique~\cite{CM96,Phi08}, we need to make sure the coloring has the property that half of point assigned $+1$ and the other half of them assigned $-1$.  We adapt a standard idea from combinatorial discrepancy~\cite{Mat10}.  
	
	This can be done by adding an all-one row to the discrepancy matrix $G$.
	It guarantees that the difference of the number of $+1$ and $-1$ is $O(\sqrt{d\log n})$ since $\gamma_2$ is a norm, and therefore we can apply the triangle inequality.
	Namely,
	\[
	\gamma_2\left(
	\begin{bmatrix}
	\mathbbm{1}_{1\times n} \\
	G
	\end{bmatrix}
	\right)
	\leq 
	\gamma_2\left(
	\begin{bmatrix}
	O_{1\times n} \\
	G
	\end{bmatrix}
	\right)
	+
	\gamma_2\left(
	\begin{bmatrix}
	\mathbbm{1}_{1 \times n} \\
	O_{m \times n}
	\end{bmatrix}
	\right)
	\]
	where $\mathbbm{1}$ is all-one matrix and $O$ is zero matrix.
	Let $P_+ = \setdef{p\in P}{\chi(p) = +1}$ and $P_- = \setdef{p\in P}{\chi(p)  =-1}$.
	Suppose there are more $+1$s than $-1$s.
	Choose $O(\sqrt{d\log n})$ points assigned $+1$ arbitrarily and flip them to $-1$ such that it makes the difference zero.
	$P_+'$ and $P_-'$ are defined in the same way as $P_+$ and $P_-$, after flipping some values.
	For any $x \in \mathbb{R}^d$,
	\begin{align*}
	&
	\abs{\sum_{p\in P_+'}K(x,p)-\sum_{p\in P_-'}K(x,p)} \\
	& \leq 
	\abs{\sum_{p\in P_+}K(x,p)-\sum_{p\in P_-}K(x,p)}+\abs{\sum_{p\in P_+'\backslash P_+}K(x,p)} + \abs{\sum_{p\in P_-\backslash P_-'}K(x,p)}  \\
	& = 
	O(\sqrt{d\log n}).  
	\end{align*}
	Now, we can apply the standard halving technique to achieve 
	\[
	\max_{x\in\mathbb{R}^d}\abs{\kde_P(x) - \kde_Q(x)} < \eps. \qedhere
	\]	
\end{proof}

%\begin{corollary}
%	For Gaussian or Laplacian kernel, for any set $P \in \mathbb{R}^d$, there is a $\eps$-kernel coreset of size $O(\frac{1}{\eps}\sqrt{d\log \frac{1}{\eps}})$.
%\end{corollary}

\paragraph{Implementation. }
Note that we do not need to decompose the entire matrix $G$. 
Instead, we just need a set of vectors $V = \setdef{v_p}{p\in P}$ such that the inner product $\langle v_{p_1}, v_{p_2} \rangle = K(p_1,p_2)$ as input to the algorithm in~\cite{bansal2017gram}.
This set $V$ can be computed in $\poly(n, d) = \poly(n)$ time assuming $d < n$.
Using the standard Merge-Reduce framework~\cite{Phi08}, the coreset with desired size can be constructed in $O(n \poly(1/\eps))$ time.

%%%%%%%%%%%%%%%%%%%%%%%%%%%%%%%%%%%%%%%%%%%%%%%%%%%%%%%%%%%%%%%%%%%
%%%%%%%%%%%%%%%%%%%%%%%%%%%%%%%%%%%%%%%%%%%%%%%%%%%%%%%%%%%%%%%%%%%
\section{Lower Bound for KDE Coreset}\label{lower}

In this section, we add two new conditions on our kernel; both of these are common properties of kernels.    

\begin{itemize}
	\item 
	A kernel $K$ is \emph{rotation- and shift-invariant} if there exists a function $f$ such that $K(x,y) = f(\|x-y\|^2)$.  
	
	\item
	A rotation- and shift-invariant kernel is \emph{somewhere $C_f$-steep} if there exist a constant $C_f > 0$, and values $z_f > r_f > 0$ such that $f(z_1) -f(z_2) > C_f \cdot(z_2 - z_1)$ for all $z_1\in (z_f - r_f , z_f )$ and $z_2 \in (z_f , z_f + r_f )$.  When $C_f$ is an absolute constant, we often just say the kernel is \emph{somewhere steep}.  
\end{itemize}

%Moreover, we assume $K$ is bounded influence mentioned in section \ref{upper}.

Phillips \cite{phillips2013varepsilon} constructed an example of $P$ of size $1/\eps$ where each point in $P$ is far away from all others.
Therefore, if one of them is not picked for a KDE coreset $Q$, the evaluation of $\kde_Q$ at that point has large error.  
%More recently, Phillips and Tai~\cite{phillips2018improved} gave another example of $P$ of size $1/\eps^2$ in an appropriately scaled simplex; that spans $\mathbb{R}^{1/\eps^2}$.  
%They showed that it produces error of $\Omega(1/\sqrt{k})$ at some point, if $k$ is the number of points picked.
%The following construction combines these two approaches.  
We divide $n=\frac{\sqrt{d}}{\eps}$ points into $n/d$ groups where each group has $d$ points that form a simplex, and each group is far away from all other groups.  
It means that there is a group producing $\Omega(1/\sqrt{d})$ error when considered alone, and then, since we have $n/d$ groups, the final error would be $\Omega(\frac{1/\sqrt{d}}{n/d})=\Omega(\eps)$.

\begin{theorem}
	\label{thm:core-lb}
	Suppose $\eps>0$.
	Consider a rotation- and shift-invariant, somewhere steep, bounded influence kernel $K$.  
	Assume $\frac{1}{\eps^2} \geq d\geq \frac{9z_f^2}{r_f^2}$, where $z_f$ and $r_f$ are absolute constants that depend on $K$ and are defined as they pertain to the somewhere steep criteria.  
	There is a set of $P\in \mathbb{R}^d$ such that, for any subset $Q$ of size $k\leq\frac{\sqrt{d}}{2\eps}$, there is a point $x \in \mathbb{R}^d$ such that $\abs{\kde_P(x)-\kde_Q(x)}> \eps$.
	%There is a set of $P\in \mathbb{R}^d$ such that we need a subset $Q$ of size $k=\Omega(\frac{\sqrt{d}}{\eps})$ to achieve $\abs{\kde_P(x)-\kde_Q(x)}\leq \eps$ for some $x\in \mathbb{R}^d$.
\end{theorem}

\begin{proof}
	Let $n=\sqrt{d}/\eps$.
	We allow weighted coresets of $Q$; that is, for each $q\in Q$, there is a real number $\beta_q$ such that $\kde_Q(x) = \sum_{q\in Q} \beta_q K(x,q)$.

	\begin{figure}
		\begin{center}
			\includegraphics[width=\textwidth]{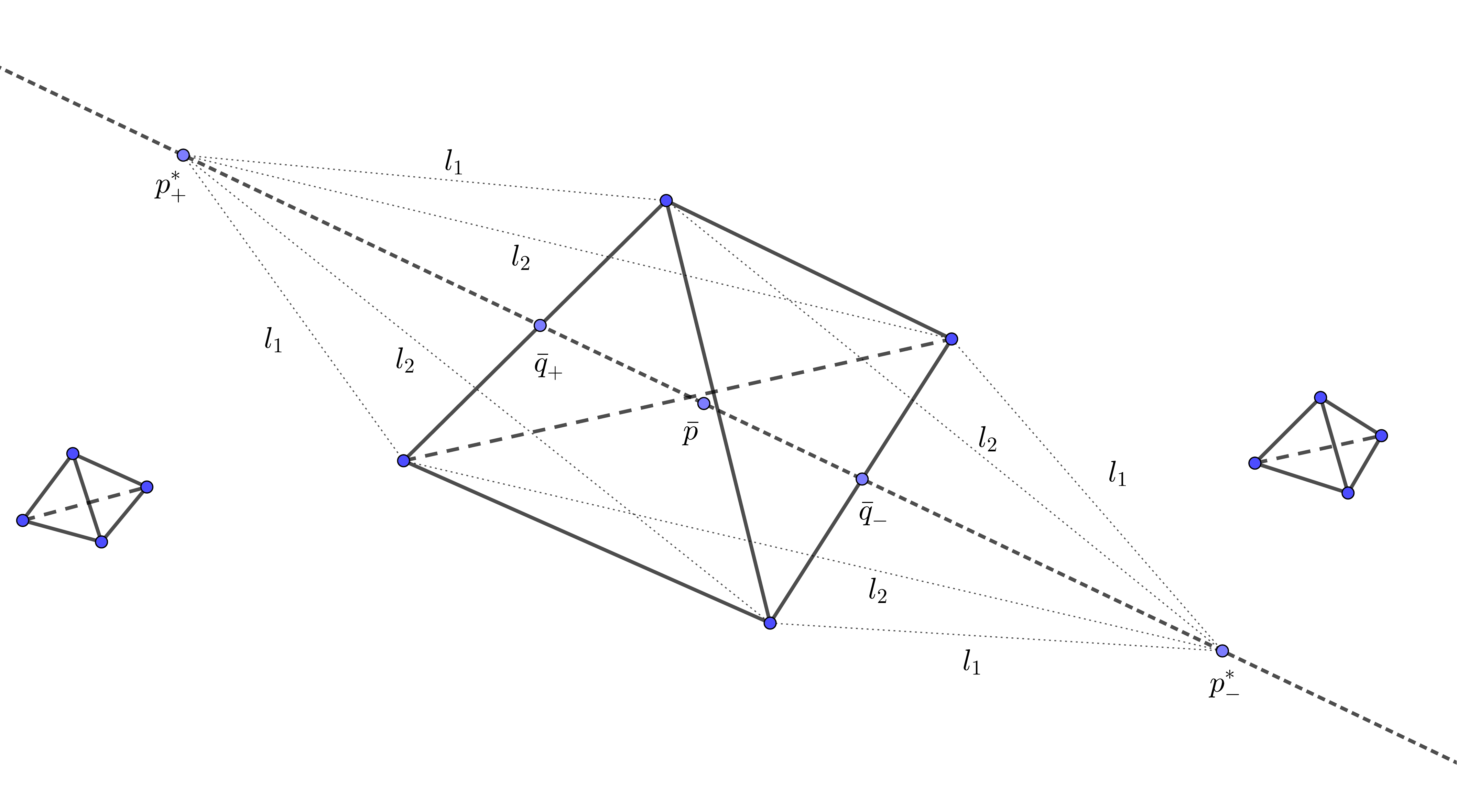}
		\end{center}
		\vspace{-.25in}
		\caption{Illustration of the lower bound construction.\label{fig:LB}}
	\end{figure}

	Let $k \leq n/2$ be the size of the potential coreset we consider.
	Construct $P$ with size of $n$ in $\mathbb{R}^d$ as follow.
	Let $\{e_i\}_{i=1}^d$ is the standard basis and $L$ is a very large number.
	Set $P_j = \setdef{p_{i,j} = \sqrt{\frac{z_f}{2}} e_i+jLe_1}{i=1,2,\dots,d}$ for all $j = 1,2,\dots,\frac{n}{d}$.
	Define $P= \cup_{j=1}^{n/d} P_j$.
	Namely, we divide $n$ points into $\frac{n}{d}$ groups and each group has $d$ points which forms a $d$-simplex.
	Also, the groups are sufficiently far away from each other.
	Suppose $Q = \cup_{j=1}^{n/d}\setdef{p_{i_a,j}}{a = 1,2,\dots,k_j}$ where $k_j$ is the number of points in $Q$ at group $j$.
	Denote $Q_j = \setdef{p_{i_a,j}}{a = 1,2,\dots,k_j}$.
	That is, $Q = \cup_{j=1}^{n/d} Q_j$ and $\abs{Q_j}=k_j\leq d$ with $\sum_{j=1}^{n/d} k_j=\abs{Q}=k$.

	Since $\sum_{j=1}^{n/d} \abs{Q_j}=k \leq n/2$, at least one $j$ must satisfy $k_j\leq \frac{d}{2}$.  
	Denote $j'$ to be that $j$.
	We can assume $k_{j'} = d/2$, otherwise, pick enough points arbitrarily from $P_j \setminus Q_{j'}$ and place them in $Q_{j'}$ to make $|Q_{j'}| = k_{j'} = d/2$, but set the corresponding weight to be $0$.  
	Denote $\bar{p} = \frac{1}{d}\sum_{p\in P_{j'}} p$ the mean of $P_{j'}$; 
	$\bar{q}_+ = \frac{2}{d}\sum_{q\in Q_{j'}}q$ the mean of $Q_{j'}$; and 
	$\bar{q}_- = \frac{2}{d}\sum_{q\in P_{j'}\backslash Q_{j'}}q$ the mean of points in $P_{j'}$ not selected into $Q_{j'}$; see Figure \ref{fig:LB}.
	Also, denote $p^*_+ = \bar{q}_+ + \sqrt{\frac{z_f}{2}}\frac{\bar{q}_+ - \bar{p}}{\norm{\bar{q}_+ - \bar{p}}}$ and 
	$p^*_- = \bar{q}_- + \sqrt{\frac{z_f}{2}}\frac{\bar{q}_- - \bar{p}}{\norm{\bar{q}_- - \bar{p}}}$; translates of these points away from the mean $\bar p$ by a specific vector.
	Note that $\norm{p^*_+-q}$ is the same for all $q\in Q_{j'}$, denoted by $l_1$ and $\norm{p^*_+-q}$ is same for all $q\in P_{j'}\backslash Q_{j'}$, denoted by $l_2$.
	By symmetry, we also have that $l_1 = \norm{p^*_--q}$ for all $q\in P_{j'}\backslash Q_{j'}$ and $l_2 = \norm{p^*_--q}$ for all $q\in  Q_{j'}$.

	If $\sum_{q\in Q_{j'}} \beta_q \geq d/n$, we evaluate the error at $p^*_+$.
	\begin{align*}
	&
	(\kde_Q-\kde_P)(p^*_+) \\
	& \qquad =
	\sum_{q\in Q_{j'}} (\beta_q-\frac{1}{n})f(\norm{p^*_+-q}^2)+\sum_{q \in P_{j'}\backslash Q_{j'}} (-\frac{1}{n})f(\norm{p^*_+-q}^2) + s \\
	& \qquad \geq
	\frac{d}{2n}(f(l_1^2)-f(l_2^2)) +s
	\end{align*}
	where $\abs{s}$ is arbitrarily small due to the choice of arbitrarily large number $L$ and the fact that $K$ is bounded influence.
	If $\sum_{q\in Q_{j'}} \beta_q \leq d/n$,  we evaluate the error at $p^*_-$.
	\begin{align*}
	&
	(\kde_P-\kde_Q)(p^*_-) \\
	& \qquad =
	\sum_{q \in P_{j'}\backslash Q_{j'}} \frac{1}{n}f(\norm{p^*_--q}^2)+\sum_{q\in Q_{j'}} (\frac{1}{n}-\beta_q)f(\norm{p^*_--q}^2) + s \\
	& \qquad \geq
	\frac{d}{2n}(f(l_1^2)-f(l_2^2)) +s
	\end{align*}
	Therefore, in either case, we need to bound $f(l_1^2)-f(l_2^2)$ from below.
	
	By direct computation, we have $l_1^2 =z_f-\frac{z_f}{d}$ and $l_2^2 = z_f+\frac{z_f}{d}+\frac{2z_f}{\sqrt{d}}$.
	By enforcing that 
	\[
	z_f-r_f < z_f-\frac{z_f}{d} = l_1^2 < z_f
	\]
	and 
	\[
	z_f < z_f+\frac{z_f}{d}+\frac{2z_f}{\sqrt{d}} = l_2^2 < z_f+\frac{3z_f}{\sqrt{d}} < z_f+r_f,  
	\]
	we can invoke the somewhere $C_f$-steep property that there exists an $x$ in $\mathbb{R}^d$ for which the inequality holds.
	Therefore,
	\[
	f(l_1^2)-f(l_2^2)
	> 
	C_f\cdot (l_2^2-l_1^2) 
	>
	C_f \cdot z_f \cdot \frac{2}{\sqrt{d}}.
	\]
	
	Hence, the error is at least 
	\[
	\frac{d}{2n}(f(l_1^2)-f(l_2^2)) +s
	>
	\frac{d}{2n} \left(	C_f \cdot z_f \cdot \frac{2}{\sqrt{d}} \right) +s
	>
	\frac{\sqrt{d}}{n} \cdot C_f \cdot z_f + s
	=
	\Omega(\sqrt{d}/n)=\Omega(\eps).	 \qedhere
	\]	
\end{proof}

Note that when $d>\frac{1}{\eps^2}$ the above argument is still valid by considering $d = \frac{1}{\eps^2}$.
Hence, we have the following conclusion.

\begin{corollary}
	Suppose $\eps>0$.
	Consider a rotation- and shift-invariant, somewhere steep, bounded influence kernel $K$.  
	Assume $d\geq \frac{9z_f^2}{r_f^2}$, where $z_f$ and $r_f$ are absolute constants that depend on $K$ and are defined as they pertain to the somewhere steep criteria.  
	There is a set of $P\subset \mathbb{R}^d$ such that, for any subset $Q$ of size $k\leq\frac{\min\{\sqrt{d},\frac{1}{\eps}\}}{2\eps}$, there is a point $x \in \mathbb{R}^d$ such that $\abs{\kde_P(x)-\kde_Q(x)}> \eps$.
	%There is a set of $P\in \mathbb{R}^d$ such that we need a subset $Q$ of size $k=\Omega(\frac{\sqrt{d}}{\eps})$ to achieve $\abs{\kde_P(x)-\kde_Q(x)}\leq \eps$ for some $x\in \mathbb{R}^d$.
\end{corollary}

%\begin{corollary}
%	For Gaussian or Laplacian kernel, there is a set $P\in \mathbb{R}^d$ such that any $\eps$-kernel coreset has size $\Omega(\frac{\sqrt{d}}{\eps})$.
%\end{corollary}

%%%%%%%%%%%%%%%%%%%%%%%%%%%%%%%%%%%%%%%%%%%%%%%%%%%%%%%%%%%%%%%%%%%
%%%%%%%%%%%%%%%%%%%%%%%%%%%%%%%%%%%%%%%%%%%%%%%%%%%%%%%%%%%%%%%%%%%
\section{Applications to Specific Positive Definite Kernels}
\label{sec:specific}

In this section, we work through the straight-forward application of these bounds to some specific kernels and settings.  

%%%%%%%%%%%%%%%%%%%%%%%%%%%%%%%%%%%%%%
\paragraph{Gaussian and Laplace kernels.}  
These kernels are defined over $\mathbb{R}^d$.  They have bounded influence, so $|K(x,p)| \leq \frac{1}{n}$ for all $p \notin [-n^{c_K}, n^{c_K}]^d$ for $c_K = 1$.  
They are also $C_K$-Lipschitz with constant $C_K = \alpha$, so $|K(x,z) - K(p,z)| \leq C_K \|x-p\|$ for any $x,p \in \mathbb{R}^d$.  
These properties imply we can invoke the discrepancy upper bound in Theorem \ref{thm:disc-up}.  

These kernels are also rotation- and shift-invariant, and somewhere steep with constant $C_f = (\alpha/2)\exp(-\alpha^2)$.  Hence we can invoke the lower bound in Theorem \ref{thm:core-lb}.

\begin{corollary} \label{cor:Gaussian}
	For Gaussian or Laplacian kernels, for any set $P \in \mathbb{R}^d$, there is a $\eps$-KDE coreset of size $O((\sqrt{d}/\eps)\sqrt{\log 1/\eps})$, and it cannot have an $\eps$-KDE coreset of size $o(\sqrt{d}/\eps)$.
\end{corollary}

The Gaussian kernel has an amazing decomposition property that in $\R^d$ if we fix any $d'$ coordinates in any way, then conditioned on those, the remaining $d-d'$ coordinates still follow a Gaussian distribution.  Among other things, this means it is useful to construct kernels for complex scenarios.  For instance, consider a large set $T$ of $n$ trajectories, each with $k$ waypoints;  e.g., backpacking or road trips or military excursions with $k$ nights, and let the waypoints be the $(x,y)$-coordinates for the location of each night stay.  We can measure the similarity between two trajectories $t = (p_1, p_2, \ldots, p_k)$ and $t' = (p'_1, p'_2, \ldots, p'_k)$ as the average similarity between the corresponding waypoints, and we can measure the similarity of any two corresponding waypoints $p_j$ and $p'_j$ with a $2$-dimensional Gaussian.  Then, by the decomposition property, the full similarity between the trajectories is precisely a $(2k)$-dimensional Gaussian.  We can thus define a kernel density estimate over these trajectories $\kde_T$ using this $(2k)$-dimensional Gaussian kernel.  
Now, given Corollary \ref{cor:Gaussian} we know that to approximate $\kde_T$ with a much smaller data set $S \subset T$ so $\|\kde_T - \kde_S\|_\infty \leq \eps$, we can construct $S$ so $|S| = O(\sqrt{k}/\eps \cdot \sqrt{\log 1/\eps})$ but cannot in general achieve $|S| = o(\sqrt{k}/\eps)$.

%%%%%%%%%%%%%%%%%%%%%%%%%%%%%%%%%%%%%%
\paragraph{Jensen-Shannon and Hellinger kernels.}
In order to apply our technique on $\Delta^d$, observe that $\Delta^d$ is a subset of a $(d-1)$-dimensional Euclidian subspace of $\R^d$; so we can simply create the grid needed for Lemma \ref{grid} within this subspace.  
Recall that these two kernel have the form of $\exp(-\alpha\mathtt{d}(x,y))$ where $\mathtt{d}(x,y)=\mathtt{d}_{\textsf{JS}}(x,y) = H(\frac{x+y}{2}) - \frac{H(x) + H(y)}{2}$ for Jensen-Shannon kernel and $\mathtt{d}(x,y) = \mathtt{d}_{\textsf{H}}(x,y) = \sum_{i=1}^d (\sqrt{x_i} - \sqrt{y_i})^2$ for Hellinger and note that $\abs{K(x,z)-K(y,z)}\leq \alpha\abs{\mathtt{d}(x,z)-\mathtt{d}(y,z)}$ for any $x,y,z \in \Delta^d$.
It is easy to estimate that when $x,y$ are sufficiently close, for JS kernel, $\abs{\mathtt{d}(x,z)-\mathtt{d}(y,z)} \leq 2d \max_i \abs{x_i-y_i}\abs{\log \abs{x_i-y_i}}\leq 2d\max_i\sqrt{\abs{x_i-y_i}}$ and for Hellinger kernel, $\abs{\mathtt{d}(x,z)-\mathtt{d}(y,z)} \leq 4d\max_i \sqrt{\abs{x_i-y_i}}$.
%In order to apply our technique on $\Delta^d$, we can rewrite the kernel to be $K'(x,y) = K(\frac{x}{\norm{x}_1}, \frac{y}{\norm{y}_1})$ for all $x,y\in \mathbb{R}_+^d$ where $\norm{x}_1 = \sum_{i=1}^d\abs{x_i}$ and $\mathbb{R}_+ = \{ x>0 \}$.
%Even though these Jensen-Shannon and Hellinger kernels are not Lipschitz, we can still modify the construction of the grid in Lemma \ref{grid} with width $\frac{1}{n^4}$ instead of $\frac{1}{\sqrt{d}n}$ such that if $x,y$ lie in same cell then $\abs{K(x,z)-K(y,z)}=O(\frac{1}{n})$ for any $x,y,z \in \Delta^d$.  \jeff{why and how?}
So even though these kernels are not Lipschitz, we can still modify the construction of the grid in Lemma \ref{grid} with width $\frac{1}{n^4}$ (assuming $d \leq n$) instead of $\frac{1}{\sqrt{d}n}$ such that if $x,y$ lie in the same cell then $\abs{K(x,z)-K(y,z)}=O(\frac{1}{n})$ for any $x,y,z \in \Delta^d$.
Since all relevant points are in a bounded domain both kernels have $c_K$-bounded influence; setting $c_K=1$ is sufficient.

\begin{corollary}\label{cor:JS+H}
	For Jensen-Shannon and Hellinger kernels, for any set $P \in \Delta^d$, there is a $\eps$-KDE coreset of size $O((\sqrt{d}/\eps)\sqrt{\log 1/\eps})$.
\end{corollary}

Note that these kernels are not rotation- and shift-invariant and therefore our lower bound result does not apply.

These kernels are based on widely-used information distances:  the 
Jensen-Shannon distance $\mathtt{d}_{\textsf{JS}}(x,p)$ 
and the 
Hellinger distance $\mathtt{d}_{\textsf{H}}(x,p)$.  
These make sense when the input data $x,p \in \Delta^d$ represent a "histogram," a discrete probability distribution over a $d$-variate domain.  
These are widely studied objects in information theory, and more commonly text analysis.  For instance, a common text modeling approach is to represent each document $v$ in a large corpus of documents $V$ (e.g., a collection of tweets, or news articles, or wikipedia pages) as a set of word counts.  That is, each coordinate $v_j$ of $v$ represents the number of times that word (indexed by) $j$ occurs in that document.
To remove length information from the documents (retaining only the topics), it is common to normalize each vector as $v \mapsto \frac{v}{\|v\|}$ so the $j$th coordinate represents the probability that a random word on the page is $j$.
The most common modeling choice to measure distance between these distribution representations of documents are the Hellinger and Jensen-Shannon distances, and hence the most natural choice of similarity are the corresponding kernels we examine.  In particular, with a very large corpus $V$ of size $n$, Corollary \ref{cor:JS+H} shows that we can approximate $\kde_V$, a kernel density estimate of $V$, with one described by a much smaller set $S \subset V$ so $\|\kde_V - \kde_S\| \leq \eps$ and so $|S| = O(\sqrt{d}/\eps \cdot \sqrt{\log 1/\eps})$.  
Noteably, when one has a fairly large $d$, and desires high accuracy (small $\eps$), then our new result will provide the best possible $\eps$-KDE coreset.

%%%%%%%%%%%%%%%%%%%%%%%%%%%%%%%%%%%%%%
\paragraph{Exponential kernels.}
In order to apply our technique on $\mathbb{S}^d$, we can rewrite the kernel to be $K'(x,y) = K(\frac{x}{\norm{x}}, \frac{y}{\norm{y}})$ for all $x,y\in \mathbb{R}^d\backslash \{0\}$.
We construct the grid in Lemma \ref{grid} on $\mathbb{R}^d$ for $K'$ and then only retain grid points which lie in the annulus $\mathbb{A}^d = \setdef{x\in\mathbb{R}^d}{\frac{1}{2}\leq \norm{x}\leq \frac{3}{2}}$.
This annulus contains all grid points which could be the closest point of some point on $\mathbb{S}^d$, as required in Lemma \ref{grid}.   Moreover $K'$ is $C_K$-Lipschitz on the annulus: it satisfies for any $x,y,z \in \mathbb{A}^d$ that $|K'(x,z) - K'(y,z)| \leq C_K \|x-y\|$, with $C_K = 4 \alpha$.  
Since the domain is restricted to $\mathbb{S}^d$, similar to on the domain $\Delta^d$,  any kernel has $c_K$-bounded influence and setting $c_K=1$ is sufficient.

%Then when constructing the set $S$ required for Lemma \ref{grid}, we can first create a set $S' = \cup_{p \in P} S_p$ where $p = R \cap [-n^{c_K}, n^{c_K}]^d$ as before.  Then we apply the normalization to all of these points $S = \{s = \frac{s'}{\|s'\|} \mid s' \in S\}$.  For any point $x \in \mathbb{S}^d$, and any point $s'$ this normalization almost always decreases the distance to $x$, and never increases it too much: $\|x-s\| < \|x-s'\| + ZZZZ$  \jeff{Waiming:  please specify $ZZZZ$ as a function of $d$ and argue it is not too big to affect results.  }  
%Thus $S$, satisfies the results for Lemma \ref{grid}.  
%\jeff{Perhaps restate a specific new version of Lemma \ref{grid}.}

\begin{corollary}
	For the exponential kernel, for any set $P \in \mathbb{S}^d$, there is a $\eps$-KDE coreset of size $O((\sqrt{d}/\eps)\sqrt{\log 1/\eps})$.
\end{corollary}

The exponential kernel is not rotation- and shift-invariant and therefore our lower bound result does not apply.

\paragraph{Sinc kernel.}
Note that the sinc kernel is not everywhere positive, and as a result of its structure the VC-dimension is unbounded, so the approaches requiring those properties~\cite{joshi2011comparing,phillips2013varepsilon} cannot be applied.
It is also not characteristic, so the embedding-based results~\cite{GBRSS12,bach2012equivalence} do not apply either.  
As a result, there is no non-trivial $\eps$-KDE coreset for the sinc kernel.  
However, in our approach, the positivity of one single entry in the discrepancy matrix does not matter so long as the entire matrix is positive definite -- which is the case for sinc.  
Therefore, our result could be applied to sinc kernel, with $c_K = 1$ (it has $1$-bounded influence), $C_K = \alpha/\pi$ (it is $(\alpha/\pi)$-Lipschitz) and $C_f = \alpha^2/2\pi^2$ (it is somewhere $(\alpha^2/2\pi^2)$-steep).

\begin{corollary}
	For sinc kernels, for any set $P \in \mathbb{R}^d$, there is a $\eps$-KDE coreset of size $O((1/\eps)\sqrt{\log 1/\eps})$ (for $d = \{1,2,3\}$), and it cannot have a $\eps$-KDE coreset of size $\Omega(1/\eps)$.  
\end{corollary}

%%%%%%%%%%%%%%%%%%%%%%%%%%%%%%%%%%%%%%%%%%%%%%%%%%%%%%%%%%%%%%%%%%%
%%%%%%%%%%%%%%%%%%%%%%%%%%%%%%%%%%%%%%%%%%%%%%%%%%%%%%%%%%%%%%%%%%%
\section{Conclusion}
\label{sec:conclude}

We proved that Gaussian kernel has a $\eps$-KDE coreset of size $O(\frac{1}{\eps}\sqrt{d\log \frac{1}{\eps}})$ 
%if $d < 1/\eps^2$ 
and the size must satisfy $\Omega(\min \{1/\eps^2, \sqrt{d}/\eps\})$; both upper and lower bound results can be extended to a broad class of kernels.  In particular the upper bounds only requires that the kernel be characteristic or in some cases only positive definite (typically the same restriction needed for most machine learning techniques) and that it has a domain which can be discretized over a bounded region without inducing too much error.  
This family of applicable kernels includes new options like the sinc kernel, which while positive definite in $\mathbb{R}^d$ for $d=\{1,2,3\}$, it is not characteristic, is not always positive, and its super-level sets do not have bounded VC-dimension.  This is the first non-trivial $\eps$-KDE coreset result for these kernels.  

By inspecting the new constructive algorithm for obtaining small discrepancy in the $\gamma_2$-norm~\cite{bansal2017gram}, the extra $\sqrt{\log }$ factor comes from the union bound over the randomness in the algorithm.
Indeed, if $d = 1/\eps^2$ then the upper bound is $O(1/\eps^2)$, which is tight.  This bound is deterministic and does not have an extra $\sqrt{ \log }$ factor.  
Therefore, a natural conjecture is that the upper bound result can be further improved to $O(\sqrt{d}/\eps)$, at least in a well-behaved setting like for the Gaussian kernel.  

There are many other even more diverse kernels which are positive definite, which operate on domains as diverse as graphs, time series, strings, and trees~\cite{HSS06}.  The heart of the upper bound construction which uses the decomposition of the associated positive definite matrix will work even for these kernels.  However, it is less clear how to generate a finite gram or discrepancy matrix $G$, whose size depends polynomially on the data set size for these discrete objects.  Such constructions would further expand the pervasiveness of the $\eps$-KDE coreset technique we present.  

%%
%% Bibliography
%%

%% Please use bibtex, 
\bibliographystyle{plain}
\bibliography{p66-phillips}

\end{document}